\documentclass[11pt]{article}
\usepackage{fullpage}
\usepackage{times}  
\usepackage{helvet}  
\usepackage{courier}  
\usepackage{algorithm}
\usepackage{algorithmic}

\usepackage{microtype}
\usepackage{graphicx}
\usepackage{subfigure}
\usepackage{booktabs} 
\usepackage{thm-restate}

\usepackage{hyperref}

\usepackage{amsmath}
\usepackage{amssymb}
\usepackage{mathtools}
\usepackage{amsthm}

\usepackage[capitalize,noabbrev]{cleveref}

\theoremstyle{plain}
\newtheorem{theorem}{Theorem}[section]

\newtheorem{observation}[theorem]{Observation}
\newtheorem{lemma}[theorem]{Lemma}

\newtheorem{claim}[theorem]{Claim}

\theoremstyle{definition}
\newtheorem{definition}[theorem]{Definition}

\newtheorem{problem}[theorem]{Problem}

\theoremstyle{remark}

\usepackage[textsize=small]{todonotes}

\usepackage{complexity}
\usepackage{enumitem}

\newcommand{\set}[1]{\left\{#1\right\}}

\newcommand{\clos}{\overline}

\newcommand{\N}{\mathbb{N}}     
\renewcommand{\R}{\mathbb{R}}     
\newcommand{\DD}{\mathcal{D}} 
\newcommand{\HH}{\mathcal{H}}
\renewcommand{\P}{\mathcal{P}}  

\newcommand{\ballcn}[3]{\clos{B}_{#1}^{#3}(#2)} 
\newcommand{\ballcinf}[2]{\ballcn{#1}{#2}{\infty}} 

\newcommand{\distribution}[3]{\mathcal{D}_{#1,#2,#3}}

\DeclareMathOperator{\range}{range}
\DeclareMathOperator{\defeq}{\overset{def}{=}}
\DeclareMathOperator{\diam}{diam}
\DeclareMathOperator{\dist}{\sim}
\DeclareMathOperator{\dtv}{d_{TV}}
\DeclareMathOperator{\bern}{Bern}
\DeclareMathOperator{\err}{e}
\DeclareMathOperator{\errunif}{\err_{unif}}

\DeclarePairedDelimiter\abs{\lvert}{\rvert}
\DeclarePairedDelimiter\norm{\lVert}{\rVert}

\newcommand*{\numberfullref}[1]{\hyperref[{#1}]{\Autoref*{#1} (\nameref*{#1})}}
\newcommand*{\namefullref}[1]{\hyperref[{#1}]{\nameref*{#1} (\Autoref*{#1})}}

\usepackage{authblk}
\makeatletter
\newcommand\email[2][]%
  {\newaffiltrue\let\AB@blk@and\AB@pand
      \if\relax#1\relax\def\AB@note{\AB@thenote}\else\def\AB@note{\relax}%
        \setcounter{Maxaffil}{0}\fi
      \begingroup
        \let\protect\@unexpandable@protect
        \def\thanks{\protect\thanks}\def\footnote{\protect\footnote}%
        \@temptokena=\expandafter{\AB@authors}%
        {\def\\{\protect\\\protect\Affilfont}\xdef\AB@temp{\url{#2}}}%
         \xdef\AB@authors{\the\@temptokena\AB@las\AB@au@str
         \protect\\[\affilsep]\protect\Affilfont\AB@temp}%
         \gdef\AB@las{}\gdef\AB@au@str{}%
        {\def\\{, \ignorespaces}\xdef\AB@temp{\url{#2}}}%
        \@temptokena=\expandafter{\AB@affillist}%
        \xdef\AB@affillist{\the\@temptokena \AB@affilsep
          \AB@affilnote{}\protect\Affilfont\AB@temp}%
      \endgroup
      \let\AB@affilsep\AB@affilsepx
}
\makeatother

\title{List and Certificate Complexities in Replicable Learning}

\renewcommand\Affilfont{\itshape\small}
\usepackage{authblk}

\author{Peter Dixon}
\affil{Ben-Gurion University of Negev}
\email{tooplark@gmail.com}

\author{A. Pavan }
\affil{Iowa State University}
\email{pavan@cs.iastate.edu}

\author{Jason Vander Woude}
\affil{University of Nebraska-Lincoln}
\email{jasonvw@huskers.unl.edu}

\author{ N. V. Vinodchandran}
\affil{University of Nebraska-Lincoln}
\email{vinod@cse.unl.edu}

\newcommand{\dbep}{{\sc $d$-Coin Bias Estimation Problem}}
\newcommand{\dtep}{{\sc $d$-Threshold Estimation Problem}}

\begin{document}
\maketitle

\begin{abstract}
We investigate replicable learning algorithms. 
Ideally, we would like to design algorithms that output the same canonical model over multiple runs, even when different runs observe a different set of samples from the unknown data distribution. In general, such a strong notion of replicability is not achievable. Thus we consider two feasible notions of replicability called {\em list replicability} and {\em certificate replicability}.  Intuitively, these notions capture the degree of (non) replicability. We design algorithms for certain learning problems that are optimal in list and certificate complexity.  We establish matching impossibility results.

\end{abstract}

\section{Introduction}

Replicability and reproducibility in science are critical concerns. The fundamental requirement that scientific results and experiments be replicable/reproducible is central to the development and evolution of science. In recent years, these concerns have grown as several scientific disciplines turn to data-driven research, which enables exponential progress through data democratization and affordable computing resources. The replicability issue has received attention from a wide spectrum of entities, from general media publications (for example, The Economist's "How Science Goes Wrong," 2013~\cite{economist}) to scientific publication venues (for example, see~\cite{Ioannidis05, Baker2016}) to professional and scientific bodies such as the National Academy of Sciences, Engineering, and Medicine (NASEM). The emerging challenges to replicability and reproducibility have been discussed in depth by a consensus study report published by NASEM~\cite{NASEM}.

A broad approach taken to ensure the reproducibility/replicability of algorithms is to make the datasets, algorithms, and code publicly available. Of late, conferences have been hosting replicability workshops to promote best practices and to encourage researchers to share code (for example, see ~\cite{PVLS21} and \cite{PSFNL2019}). An underlying assumption is that consistent results can be obtained using the same input data, computational methods, and code. However, this practice alone is insufficient to ensure replicability as modern-day approaches use computations that inherently involve randomness.

Computing over random variables results in a high degree of non-replicability, especially in machine learning tasks. Machine learning algorithms observe samples from a (sometimes unknown) distribution and output a hypothesis or model. Such algorithms are inherently non-replicable. Two distinct runs of the algorithm will output different models as the algorithms see different sets of samples over the two runs. Ideally, to achieve "perfect replicability," we would like to design algorithms that output the same canonical hypothesis over multiple runs, even when different runs observe a different set of samples from the unknown distribution.

We first observe that perfect replicability is not achievable in learning, as  a dependency of the output on the data samples is inevitable. We illustrate this with a simple learning task of estimating the bias of a coin: given $n$ independent tosses from a coin with unknown bias $b$, output an estimate of $b$. It is relatively easy to argue that there is no algorithm that  outputs a {\em canonical estimate} $v_b$ with probability $\geq 2/3$
so that $|v_b-b| \leq \varepsilon$. Consider a sequence of coins with biases $b_1 < b_2 < \cdots < b_m$ where each $b_{i+1}-b_i \leq \eta$  (for some small $\eta$) but $b_m-b_1 \geq 2\varepsilon$. For two adjacent biases, the statistical distance (denoted by $\dtv$) between $\mathcal{D}_{i+1}^n$ and  $\mathcal{D}_{i}^n$ is $\leq m\eta$, where $\DD_i^n$ is the distribution of $n$ independent tosses of the $i^{th}$ coin.
Let $v_{i+1}$ and $v_i$ be the canonical estimates output by the algorithm for biases $b_{i+1}$ and $b_i$ respectively. Since $A$ on samples from distribution $\mathcal{D}_{i}^n$ outputs $v_{i}$ with probability at least $2/3$ and $\dtv(D_i^n, D_{i+1}^n) \leq n \eta$,  $A(\mathcal{D}_{i+1}^n)$ must output $v_i$  with probability at least $2/3-n\eta$. We take $\eta = 1/10n$. Since $A(\mathcal{D}_{i}^n)$ must output a canonical value with probability at least $2/3$, this implies that $v_i=v_{i+1}$.  Thus, on all biases, $b_1, \ldots, b_m$, $A$ should output the same value and this leads to a contradiction since $b_1$ and $b_m$ are $2\varepsilon$ apart. However, it is easy to see that there is an algorithm for bias-estimation that outputs one of {\em two canonical} estimates using $n=O(1/\varepsilon^2)$ tosses: estimate the bias within $\varepsilon/2$ and round the value to the closest multiple of $\varepsilon$. The starting point of our work is these two observations. Even though it may not be possible to design learning algorithms that are perfectly replicable, it is possible to design algorithms whose ``non-replicability'' is minimized.

Motivated by this, we study two notions of replicability called {\em list-replicability} and {\em certificate-replicability} in the context of machine learning algorithms. These notions {\em quantify} the degree of (non)-replicability.  The notions we consider are rooted in the {\em pseudodeterminism}-literature~\cite{GatGoldwasser11,Goldreich19,GrossmanLiu19} which has been an active area in randomized algorithms and computational complexity theory. Section~\ref{sec:prior} discusses this connection and other works that are related to our work.

\subsection{Our Results}
We consider two notions of replicable learning: list-replicable learning and certificate-replicable learning. In list replicability, the learning algorithm should output one of the models from a list of models of size $\leq k$ with high probability. This means that when the learning algorithm is run multiple times, we see at most $k$ distinct models/hypotheses.   The value $k$ is called the {\em list complexity} of the algorithm. An algorithm whose list complexity is $1$ is perfectly replicable.
Thus list complexity can be considered as the degree of (non) replicability.  The goal in this setting is to design learning algorithms that minimize the list complexity $k$.

In certificate replicability, the learning algorithm has access to an $\ell$-bit random string called the certificate of replicability (that is independent of samples and the other randomness that it may use). We require that for most certificates, the algorithm must output a {\em canonical model} $h$ that can depend only on this $\ell$-bit random string $r$. Thus once we fix a certificate $r$, multiple runs of the algorithm that have access to $r$ will output the same model w.h.p.  We call $\ell$ the {\em certificate complexity}. Note that an algorithm with zero certificate complexity is perfectly replicable. Thus $\ell$ is another measure of the degree of (non) replicability of the algorithm. The goal in this setting is to design learning algorithms that minimize $\ell$. 

A critical resource in machine learning tasks is the number of samples that the algorithm observes known as the {\em sample complexity}. An efficient learning algorithm uses as few samples as possible.  In this work, we  measure the efficiency of learning algorithms in terms of sample complexity, list complexity, and certificate complexity. This work initiates a study of learning algorithms that are efficient in certificate/list complexities as well as 
sample complexity. We establish both positive results and impossibility results for certain learning tasks.
\paragraph{Estimating the bias of $d$ coins.}
Our first set of results is on efficient replicable algorithms for the coin-bias estimation problem. We consider the problem of estimating the biases of $d$ coins simultaneously by observing $n$ tosses of each of the coins which we call \dbep. The task is to output a bias vector $\vec{v}$ so that $\norm{\vec{b}-\vec{v}}_{\infty} \leq \varepsilon$  where $\vec{b}$ is the true bias vector. We show that there is a $(d+1)$-list replicable learning algorithm for this problem with a sample complexity (number of coin tosses) $n=O({\frac{d^2}{\varepsilon^2}}\cdot\log{\frac{d}{\delta}})$ per coin.   We also design a $\lceil \log {d\over \delta} \rceil$-certificate reproducible algorithm for the problem with sample complexity $O({d^2\over \varepsilon^2\delta^2})$ per coin. Here $(1-\delta)$ is the success probability. 

We also establish the optimality of the above upper bounds in terms of list and certificate complexities. We show that there is no $d$-list replicable learning algorithm for \dbep.  This leads to a lower bound of $\Omega(\log d)$ on its certificate complexity. For establishing this impossibility result we use a version of KKM/Sperner's Lemma.   
\paragraph{PAC learning.} We establish possibility and impossibility results for list/certificate replicable algorithms in the context of PAC learning. We show that any concept class that can be learned using $d$ non-adaptive statistical queries can be learned by a $(d+1)$-list replicable PAC learning algorithm with sample complexity $O({d^2 \over \nu^2}\cdot \log {d\over \delta})$ where $\nu$ is the statistical query parameter.  We also show that such concept classes admit a $\lceil \log {d\over \delta}\rceil$-certificate replicable PAC learning algorithm with sample complexity $O({d^2\over \nu^2\delta^2}\cdot \log{d\over \delta})$. Finally, we establish tight results in the PAC learning model. In particular, we prove that the concept class of $d$-dimensional thresholds does not admit a $d$-list replicable learning algorithm under the uniform distribution. Since we can learn $d$-dimensional thresholds under the uniform distribution, using $d$ non-adaptive statistical queries, we get a $d+1$-list replicable PAC algorithm under the uniform distribution. This yields matching upper and lower bounds on the list complexity of PAC learning of thresholds under the uniform distribution.

\section{Prior and Related Work}\label{sec:prior}
Formalizing reproducibility and replicability has gained considerable momentum in recent years. While the terms reproducibility and replicability are very close and often used interchangeably, there has been an effort to distinguish between them and accordingly, our notions fall in the replicability definition~\cite{PVLS21}. 

In the context of randomized algorithms, various notions of reproducibility/replicability have been investigated. The work of Gat and Goldwasser~\cite{GatGoldwasser11} formalized and defined the notion of {\em pseudodeterministic algorithms}. A randomized algorithm $A$ is {\em pseudodeterministic} if, for any input $x$, there is a canonical value $v_x$ such that $\Pr[A(x) = v_x] \geq 2/3$. Gat and Goldwasser designed polynomial-time pseudodeterministic algorithms for algebraic computational problems, such as finding quadratic non-residues and finding non-roots of multivariate polynomials~\cite{GatGoldwasser11}. Later works studied the notion of pseudodeterminism in other algorithmic settings, such as parallel computation, streaming and sub-linear algorithms, interactive proofs, and its connections to complexity theory~\cite{GoldwasserGrossman17,GoldwasserGrossmanHolden18, OliveiraSanthanam17, OliveiraSanthanam18,AV20,GGMW20,LOS21,DPWV22}.

In the algorithmic setting, mainly two generalizations of pseudodeterminism have been investigated: {\em multi-pseudodeterministic algorithms}~\cite{Goldreich19} and {\em influential bit algorithms}~\cite{GrossmanLiu19}. A randomized algorithm $A$ is $k$-pseudodeterministic if, for every input $x$, there is a set $S_x$ of size at most $k$ such that the output of $A(x)$ belongs to the set $S_x$ with high probability. When $k=1$, we get pseudodeterminism.  A randomized algorithm $A$ is $\ell$-influential-bit algorithm if, for every input $x$, for most of the strings $r$ of length $\ell$, there exists a canonical value $v_{x,r}$ such that the algorithm $A$ on inputs $x$ and $r$ outputs $v_{x,r}$ with high probability. The string $r$ are called  the {\em influential bit} string.  Again, when $\ell=0$, we get back pseudodeterminism. The main focus of these works has been to investigate reproducibility in randomized search algorithms.

Very recently, pseudodeterminism and its generalizations have been explored in the context of learning algorithms to formalize the notion of replicability. In particular, the work of Impagliazzo, Lei, Pitassi, and Sorrell \cite{impagliazzo_reproducibility_2022} introduced the notion of {\em $\rho$-replicability}. A learning algorithm $A$ is $\rho$-replicable if $\Pr[A(S_1, r) = A(S_2, r)] \geq 1- \rho$, where $S_1$ and $S_2$ are samples drawn from a distribution $\mathcal{D}$ and $r$ is the internal randomness of the learning algorithm $A$. They designed replicable algorithms for many learning tasks, including statistical queries, approximate heavy hitters, median, and learning half-spaces.

Another line of recent work connects replicable learning to differentially private learning.
In a recent seminal work, Bun, Livny, and Moran~\cite{BLM20} showed that every concept class with finite Littlestone dimension can be learned by an approximate differentially private algorithm. This, together with an earlier work of Alon, Livny, Malliaris, and Moran~\cite{ALMM19}, establishes an equivalence between online learnability and differentially private PAC learnability. Rather surprisingly, the proof of~\cite{BLM20} uses the notion of "global stability," which is similar to the notion of pseudodeterminism in the context of learning. They define a learning algorithm $A$ to be $(n,\eta)$-globally stable with respect to a distribution $D$ if there is a hypothesis $h$ such that $\Pr_{S\sim D^n}(A(S)=h) \geq \eta$. They showed that any concept class with Littlestone dimension $d$ has an $(n,\eta)$-globally stable learning algorithm with $m = \tilde{O}(2^{2^d}/\alpha)$ and $\eta = \tilde{O}(2^{-2^d})$, where the error of $h$ (with respect to the unknown hypothesis) is $\leq \alpha$. Then they established that a globally stable learner implies a differentially private learner. The notion of globally stable learning is the perfect replicability that we discuss in the introduction when $\eta = 2/3$. Thus, as discussed in the introduction, it follows that designing globally stable algorithms with $\eta > 1/2$ is not possible, even for the simple task of estimating the bias of a coin.
The work of Ghazi, Kumar, and Manurangsi~\cite{GKM21} extended the notion of global stability to pseudo-global stability and list-global stability. The notion of pseudo-global stability is very similar to the earlier-mentioned notion of influential bit algorithms of Grossman and Liu~\cite{GrossmanLiu19} when translated to the context of learning. Similarly, the list-global stability is similar to Goldreich's notion of multi-pseudodeterminism~\cite{Goldreich19}. It is also known that the notions of pseudo-global stability and $\rho$-replicability are the same up to polynomial factors in the parameters~\cite{impagliazzo_reproducibility_2022, GKM21}. The work of~\cite{GKM21} uses these notions to design user-level private learning algorithms.

Our notion of {\em list replicability} is inspired by the notion of multi-pseudodeterminism and the notion of {\em certificate replicability} is inspired by the notion of influential-bit algorithms. In the learning setting, our notion of list replicability is similar to the notion of list-global stability and the notion of certificate replicability is similar to the notion of pseudo-global stability which in turn is similar to the notion of $\rho$-replicability of~\cite{impagliazzo_reproducibility_2022}.

We introduce the notions of list and certificate complexities that measure the {\em degree of (non) replicability}. Our goal is to design learning algorithms with optimal list and certificate complexities while minimizing the sample complexity. The earlier works did not focus on minimizing these quantities. The works of~\cite{BLM20,GKM21} used replicable algorithms as an intermediate step to design differentially private algorithms. The work of~\cite{impagliazzo_reproducibility_2022} did not consider reducing the certificate complexity in their algorithms and also did not study list-replicability

The study of various notions of reproducibility/replicability in various computational fields is an emerging topic.  In ~\cite{replicable_bandits}, the authors consider replicability in the context of stochastic bandits. Their notion is similar to the notion studied in ~\cite{impagliazzo_reproducibility_2022}.  In~\cite{reproducibility_optimization}, the authors investigate reproducibility\footnote{See \cite{PVLS21} for a discussion on replicability and replicability.} in the context of optimization with {\em inexact oracles} (initialization/gradient oracles). The setup and focus of these works are different from ours.

\section{Primary Lemmas}

In this section, we state a few lemmas that build on the work of \cite{geometry_of_rounding_arxiv} and \cite{de_loera_polytopal_2002} that will be useful for algorithmic constructions and impossibility results in the remainder of the paper.

\subsection{Partitions and Rounding}
In this subsection, we define a deterministic rounding function that we will use repeatedly. This rounding function is based on the notion of secluded partitions defined in the work of~\cite{geometry_of_rounding_arxiv}. 

We will use the following notation. We use $\diam_\infty$ to indicate the diameter of a set relative to the $\ell_\infty$ norm and $\ballcinf{\epsilon}{\vec{p}}$ to represent the closed ball of radius $\epsilon$ centered at $\vec{p}$ relative to the $\ell_\infty$ norm. That is, in $\R^d$ we have $\ballcinf{\epsilon}{\vec{p}}=\prod_{i=1}^d[p_i-\epsilon,p_i+\epsilon]$. 

Let $\mathcal{P}$ be a partition of $\mathbb{R}^d$. For a point $\vec{p} \in \mathbb{R}^d$, we use $N_{\epsilon}(\vec{p})$ to denote the set 
\[\{X \in \mathbb{P}~|~B_{\epsilon}(\vec{p}) \cap X \neq \emptyset\}\]

\begin{definition}
Let $\mathcal{P}$ be a partition of $\mathbb{R}^d$. We say that $\mathcal{P}$ is $(k, \epsilon)$-secluded, if for every point $\vec{p} \in \mathbb{R}^d$, $|N_{\epsilon(\vec{p}})| \leq k$.
\end{definition}

The following theorem is from~\cite{geometry_of_rounding_arxiv}.

\begin{theorem}\label{thm:partition}
For each $d\in\N$, there exists a $(d+1,\frac{1}{2d})$-secluded  partition, where each member of the partition is a unit hypercube.  Moreover, the partition is efficiently computable:
    Given an arbitrary point $\vec{x}\in\R^d$, the description of the partition member to which $\vec{x}$ belongs can be computed in time polynomial in $d$.
\end{theorem}

\begin{definition}[$(k(d),\epsilon(d))$-Deterministic Rounding]
A deterministic rounding scheme is a family of functions ${\mathcal F} = \{f_d\}_{d\in\N}$ where $f_d: \R^d \to \R^d$. We call $\mathcal{F}$ a  $(k(d),\varepsilon(d))$-deterministic rounding scheme if (1)  $\forall \vec{x}\in\R^d$, $d_{max}(\vec{x},f_d(\vec{x}))\leq 1
\footnote{
  The bound of 1 is not critical. We can use any constant and scale the parameters appropriately.
}$
(2) $\forall \vec{x}\in \R^d$, $\abs{\{f_d(\vec{z})\colon \vec{z}\in B_{\varepsilon(d)}(\vec{x})\}} \leq k(d)$. 

\end{definition}

The following observation is from~\cite{geometry_of_rounding_arxiv}.

\begin{observation}[Equivalence of Rounding Schemes and Partitions]\label{obs:rounding-schemes-and-partitions}
A $(k(d), \epsilon(d))$-deterministic rounding scheme induces, for each $d\in\N$, a $(k(d), \epsilon(d))$-secluded partition of $\R^d$ in which each member has diameter at most $2$. Conversely, a sequence $\langle \P_d\rangle_{d=1}^\infty$ of partitions where $\P_d$ is $(k(d), \epsilon(d))$-secluded and contains only members of diameter at most $1$ induces a $(k(d), \epsilon(d))$-deterministic rounding schemes.
\end{observation}

\subsection{A Universal Rounding Algorithm for List Replicability}

In this subsection, we will design a deterministic algorithm that will be used as a sub-routine in our list-replicable algorithms.

\begin{lemma}\label{universal_rounding_function}
Let $d\in\N$ and $\epsilon\in(0,\infty)$. Let $\epsilon_0=\frac{\epsilon}{2d}$. There is an efficiently computable function $f_{\epsilon}:\R^d\to\R^d$ with the following two properties: 
\begin{enumerate}
    \item For any $x\in\R^d$ and any $\hat{x}\in\ballcinf{\epsilon_0}{x}$ it holds that $f_{\epsilon}(\hat{x})\in\ballcinf{\epsilon}{x}$.
    \item For any $x\in\R^d$ the set $\set{f_{\epsilon}(\hat{x})\colon \hat{x}\in\ballcinf{\epsilon_0}{x}}$ has cardinality at most $d+1$.
\end{enumerate}
Informally, these two conditions are (1) if $\hat{x}$ is an $\epsilon_0$ approximation of $x$, then $f_{\epsilon}(\hat{x})$ is an $\epsilon$ approximation of $x$, and (2) $f_{\epsilon}$ maps every $\epsilon_0$ approximation of $x$ to one of at most $d+1$ possible values.
\end{lemma}

\begin{proof}
Let $\P$ be a $(d+1, \frac{1}{2d})$-secluded partition from Theorem~\ref{thm:partition} and $f:\R^d\to\R^d$ the associated deterministic rounding function due to Observation~\ref{obs:rounding-schemes-and-partitions} The partition $\P$ consists of translates of the unit cube $[0,1)^d$ with the property that for any point $\vec{p}\in\R^d$ the closed cube of side length $1/d$ centered at $\vec{p}$ (i.e. $\ballcinf{1/2d}{\vec{p}}$) intersects at most $d+1$ members/cubes in $\P$. The associated rounding function $f:\R^d\to\R^d$ maps each point of $\R^d$ to the center point of the unique cube in $\P$ which contains it. This means $f$ has the following two properties (which are closely connected to the two properties we want of $f_{\epsilon}$): (1) for every $a\in\R^d$, $\norm{f(a)-a}_\infty\leq\frac12$ (because every point is mapped to the center of its containing unit cube) and (2) for any point $p\in\R^d$, the set $\set{f(a)\colon a\in\ballcinf{1/2d}{p}}$ has cardinality at most $d+1$ (because $\ballcinf{1/2d}{p}$ intersects at most $d+1$ members of $\P$). Define $f_{\epsilon}:\R^d\to\R^d$ by $f_{\epsilon}(a)=\epsilon\cdot f(\frac{1}{\epsilon}a)$. The efficient computability of $f_\epsilon$ comes from the efficient computability of $f$ (i.e. the ability to efficiently compute the center of the unit cube in $\P$ which contains a given point).

To see that $f_\epsilon$ has property (1), let $x\in\R^d$ and $\hat{x}\in\ballcinf{\epsilon_0}{x}$. Then we have the following (justifications will follow):
\begin{align*}
    \norm*{\tfrac1\epsilon\cdot f_\epsilon(\hat x) - \tfrac1\epsilon x}_\infty 
    &= \norm*{f(\tfrac1\epsilon\hat x) - \tfrac1\epsilon x}_\infty \\
    & \leq \norm*{f(\tfrac1\epsilon\hat x) - \tfrac1\epsilon \hat x}_\infty + \norm*{\tfrac1\epsilon \hat x - \tfrac1\epsilon x}_\infty \\
    & \leq \norm*{f(\tfrac1\epsilon\hat x) - \tfrac1\epsilon \hat x}_\infty + \tfrac1\epsilon \norm*{\hat x - x}_\infty \\
    & \leq \tfrac12 + \tfrac1\epsilon \epsilon_0 \\
    & = \tfrac12 + \tfrac1{2d}  \leq 1
\end{align*}
The first line is by the definition of $f_\epsilon$, the second is the triangle inequality, the third is scaling of norms, the fourth uses the property of $f$ that points are not mapped a distance more than $\frac12$ along with the hypothesis that $\hat{x}\in\ballcinf{\epsilon_0}{x}$, the fifth uses the definition of $\epsilon_0$, and the sixth uses the fact that $d\geq1$.

Scaling both sides by $\epsilon$ and using the scaling of norms, the above gives us $\norm*{f_\epsilon(\hat x) - x}_\infty \leq \epsilon$ which proves property (1) of the lemma.

To see that $f_\epsilon$ has property (2), let $x\in\R^d$. We have the following set of equalities:
\begin{align*}
    \set{f_\epsilon(\hat x)\colon \hat x \in \ballcinf{\epsilon_0}{x}}   
    &= \set{\epsilon \cdot f(\tfrac1\epsilon \hat x)\colon \hat x \in \ballcinf{\epsilon_0}{x}} \\
    &= \set{\epsilon \cdot f(a)\colon a\in \ballcinf{\tfrac1\epsilon \epsilon_0}{x}} \\
    &= \set{\epsilon \cdot f(a)\colon a\in \ballcinf{\tfrac1{2d}}{x}}
\end{align*}
The first line is from the definition of $f_\epsilon$, the second is from re-scaling, and the third is from the definition of $\epsilon_0$.

Because $f$ takes on at most $d+1$ distinct values on $\ballcinf{\tfrac1{2d}}{x}$, the set has cardinality at most $d+1$ which proves property (2) of the lemma.
\end{proof}

\subsection{A Universal Rounding Algorithm for Certificate Replicability}

For designing certificate replicable learning algorithms we will use a general randomized procedure which is stated in the following lemma. This lemma uses a randomized rounding scheme. Similar randomized rounding schemes have been used in a few prior works~\cite{SaksZhou99,DixonPavanVinod18,Goldreich19, GrossmanLiu19, impagliazzo_reproducibility_2022}.

\begin{lemma}\label{lemma:certificate}
  Let $d\in\N$, $\epsilon_0\in(0,\infty)$ and $0<\delta<1$. There is an efficiently computable deterministic function $f:\{0,1\}^{\ell}\times \R^d\to\R^d$ with the following property. 
For any $x\in \R^d$, 
\[
\Pr_{r\in \{0,1\}^\ell}\left [ \exists x^* \in \ballcinf{\varepsilon}{x}~~ \forall \hat{x} \in \ballcinf{\varepsilon_0}{x}:f(r,\hat{x}) = x^* \right ] \geq 1-\delta
\]
where $\ell=\lceil \log {d\over \delta} \rceil$ and $\varepsilon = (2^\ell+1)\epsilon_0 \leq {2\varepsilon_0d\over \delta}$.     
\end{lemma}
\begin{proof}
Partition each coordinate of $\R^d$ into $2\varepsilon_0$-width intervals. The algorithm computing the function $f$ does the following simple randomized rounding: 
 
{\em The function $f:$} Choose a random integer $r \in \{1\dots 2^\ell\}$. Note that $r$ can be  represented using $\ell$ bits. Consider the $i^{th}$ coordinate of ${\hat{x}}$ denoted by ${\hat{x}[i]}$. Round $\hat{x}[i]$ to the nearest $k*(2\varepsilon_0)$ such that $k\mod 2^\ell \equiv r$. 
    
Now we will prove that $f$ satisfies the required properties.     
    
First, we prove the approximation guarantee. Let ${x'}$ denote the point in $\R^d$ obtained after rounding each  coordinate of $\hat{x}$. 
The $k$s satisfying $k\mod 2^\ell\equiv r$ are $2^\ell\cdot 2\varepsilon_0$ apart. Therefore, ${x'}[i]$ is rounded by at most $2^\ell\epsilon_0$. That is, $|{x'}[i]- \hat{x}[i]|  \leq 2^\ell\epsilon_0 = {\varepsilon_0 d \over \delta}$  for every $i$, $1 \leq i \leq d$.
Since $\hat{x}$ is an $\varepsilon_0$-approximation (i.e. each coordinate $\hat{x}[i]$ is within $\varepsilon_0$ of the true value $x[i]$), then each coordinate of ${x'}$ is within $(2^\ell+1)\varepsilon_0$ of $x[i]$. Therefore ${x'}$ is a $(2^\ell+1)\varepsilon_0$-approximation of $x[i]$. 
Thus $x' \in \ballcinf{\varepsilon}{x}$ for any choice of $r$.  

Now we establish that for $\geq 1-\delta$ fraction of $r \in \{1 \ldots 2^{\ell}\}$,  there exists $x^*$ such  every $\hat{x}\in \ballcinf{\varepsilon_0}{x}$ is rounded $x^*$. We argue this with respect to each coordinate and apply the union bound. Fix an $x$ and a coordinate $i$. For $x[i]$, consider the $\varepsilon_0$ interval around it.  

Consider $r$ from $\{1 \ldots 2^\ell\}$. When this $r$ is chosen, then we round $\hat{x}[i]$ to the closest $k*(2\varepsilon_0)$ such that $k \mod 2^\ell \equiv r$. Let $p^r_1, p^r_2, \ldots p^r_{j} \ldots$ be the set of such points: more precisely $p_j=(j2^l+r) * 2\varepsilon_0$. Note that $\hat{x}[i]$ is rounded to an $p_j$ to some $j$.  Let $m^r_j$ denote the midpoint between $p^r_j$ and $p^r_{j+1}$. I.e, $m^r = (p^r_j+p^r_{j+1})/2$
 We call $r$ `bad' for $x[i]$ if $x[i]$ is close to some $m^r_j$.
That is, $r$ is `bad' if $|x[i]-m^r_j| < \varepsilon_0$. Note that 
for a bad $r$ there exists $\hat{x_1}$ and $\hat{x_2}$ in $\ballcinf{\varepsilon_0}{x}$ so that their $i^{th}$ coordinates are round to $p^r_j$ and $p^r_{j+1}$ respectively. 
The crucial point is that if $r$ is `not bad' for $x[i]$,  then for every $x' \in \ballcinf{\varepsilon_0}{x}$, there exists a canonical $p^*$ such that $x'[i]$ is rounded to $p^*$.
We call $r$ bad for $x$, if $r$ is bad for $x$, if there exists at least one $i$, $1 \leq i \leq d$ such that $r$ is bad for $x[i]$. With this, it follows that if $r$ is not bad for $x$, then  there exists a canonical $x^*$ such that every $x' \in \ballcinf{\varepsilon_0}{x}$ is rounded to $x^*$.

With this, the goal is to bound the probability that a randomly chosen $r$ is bad for $x$. For this, 
we first bound the probability that $r$ is bad for $x[i]$. We will argue that there exists almost one  bad $r$ for $x[i]$.  Suppose that there exist two numbers $r_1$ and $r_2$ that are both bad for $x[i]$. 
This means that $|x[i]-m^{r_1}_{j_1}| < \varepsilon_0$ and $|x[i]-m^{r_2}_{j_2}| < \varepsilon_0$ for some $j_1$ and $j_2$. 
Thus by triangle inequality $|m^{r_1}_{j_1} -m^{r_2}_{j_2}| < 2\varepsilon_0$. However, note that $|p^{r_1}_{j_1}- p^{r_2}_{j_2}|$ is $|(j_1-j_2)2^\ell + (r_1 - r_2)|2\epsilon_0$. Since $r_1 \neq r_2$, this value is at least $2\varepsilon_0$. This implies that the absolute value of  difference between $m^{r_1}_{j_1}$ and $m^{r_2}_{j_2}$ is at least $2\varepsilon$ leading to a contradiction.

Thus the probability that $r$ is bad for $x[i]$ is atmost $\frac{1}{2^\ell}$ and by the union bound the probability that $r$ is bad for $x$ is atmost $\frac{d}{2^\ell} \leq \delta$.
This completes the proof.
\end{proof}

\subsection{A Consequence of Sperner's Lemma/KKM Lemma}
The following result is a corollary to some cubical variants of Sperner's lemma/KKM lemma initially developed in \cite{de_loera_polytopal_2002} and expanded on in \cite{geometry_of_rounding_arxiv}. The statement and proof of this result is quite similar to that of Theorem 9.4 (Second Optimality Theorem) in \cite{geometry_of_rounding_arxiv} except that it is stated here for $[0,1]^d$ instead of $\R^d$. 

\begin{lemma}\label{cubical_sperner_partition}
Let $\P$ be a partition of $[0,1]^d$ such that for each member $X\in\P$, it holds that $\diam_\infty(X)<1$. Then there exists $\vec{p}\in[0,1]^d$ such that for all $\delta>0$ we have that $\ballcn{\delta}{\vec{p}}{\infty}$ intersects at least $d+1$ members of $\P$.
\end{lemma}

\section{Replicability of  Learning Coins Biases}
In this section, we establish replicability results for estimating biases of $d$ coins.  

\begin{definition}The \dbep\ is the following problem: Design an algorithm $A$ (possibly randomize) that given $\epsilon\in(0,\infty)$, $\delta\in(0,1]$, and $n$ independent tosses (for each coin) of  an ordered collection of $d$-many biased coins with a bias vector $\vec{b} \in [0,1]^d$ outputs $\vec{v}$ so that $\norm{\vec{b}-\vec{v}}_{\infty} \leq \varepsilon$ with probability $\geq 1-\delta$.  
\end{definition}

\begin{definition} We say an algorithm $A$ for \dbep\ is $k$-{\em list replicable}, if for any bias vector $\vec{b} \in [0,1]^d$, and parameters $\varepsilon, \delta$,  
there is set $L \subseteq \ballcinf{\varepsilon}{\vec{b}}$ and an integer $n$ such that $|L|\leq k$ and $A$ on input $\varepsilon$ and $\delta$ and $n$ independent tosses (per coin) according to the bias vector $\vec{b}$, outputs an estimate $\vec{v} \in L$, with probability $\geq 1-\delta$.  $n$ is the sample complexity of $A$ and $k$ is the list complexity of $A$.      
\end{definition}

\begin{definition}
  We say an algorithm $A$ for \dbep\ is $\ell$-{\em certificate replicable}, if for any bias vector $\vec{b} \in [0,1]^d$, and parameters $\varepsilon, \delta$:  
$A$ on inputs $\epsilon$, $\delta$, $r \in \{0, 1\}^{\ell}$, and $n$ independent coin tosses (per coin) according to $\vec{b}$ satisfy the following: 
 \[\Pr_{r \in \{0, 1\}^\ell}\left[ \exists \vec{v}_r \in \ballcinf{\varepsilon}{\vec{b}}: \Pr[A \mbox{ outputs }\vec{v_r}] \geq 1-\delta\right] \geq 1-\delta\]

 In the above, the inner probability is taken over the internal randomness of $A$ and the  coin tosses. Algorithm $A$ also gets $r$ as an input (in addition to the other inputs).
We call $n$ the {\em sample complexity} and the number of random  bits $\ell$ the {\em certificate complexity} of $A$. 
\end{definition}

We note that from a coarse sense, we can convert list replicable algorithms to certificate replicable algorithms and vice-versa. However, such transformations will result in a degradation of sample complexity which is a concern of this paper.

In the following, the output of an algorithm for \dbep\ is denoted as 
$\distribution{A}{\vec{b}}{n}$.

\subsection{Replicable Algorithms}

We present two algorithms for \dbep. The first one $d+1$-list replicable and the second one is $\lceil \log {d\over \delta}\rceil $-certificate replicable.  

\begin{theorem}\label{yes_d_plus_1_pd_for_dbep}
There exists an $(d+1)$-list replicable algorithm for \dbep\  with sample complexity 
$n=O({d^2\over \varepsilon^2} \cdot \log {d\over \delta})$ (per coin).
\end{theorem}

\begin{algorithm}[h]
   \caption{$(d+1)$-list replicable algorithm for \dbep\ as in \autoref{yes_d_plus_1_pd_for_dbep}}
   \label{dbep_algorithm}
\begin{algorithmic}
   \STATE {\bfseries Input:} $\epsilon>0$
   \STATE {\bfseries \phantom{Input:}} $\delta\in(0,1]$
   \STATE {\bfseries \phantom{Input:}} sample access to $d$ coins with biases $\vec{b}\in[0,1]^d$
   \STATE {\bfseries Output:} The algorithm behaves as a $(d+1)$-pseudodeterministic $(\epsilon,\delta)$-approximation of $\vec{b}$ and is guaranteed to return a value in $[0,1]^d$.
   \STATE {\bfseries Algorithm:}
   \STATE $\epsilon_0 \defeq \frac{\epsilon}{2d}$
   \STATE $\delta_0 \defeq \frac{\delta}{d}$
   \STATE $n \defeq O\left(\frac{\ln(1/\delta_0)}{\epsilon_0^2}\right) = O\left( \frac{d^2\ln(d/\delta)}{\epsilon^2} \right)$ for some constant
   \STATE Let $f_\epsilon:\R^d\to\R^d$ be as in \autoref{universal_rounding_function}.
   \STATE Let $g:\R^d\to[0,1]^d$ be the function which restricts coordinates to the unit interval (i.e. \[g(\vec{y})\defeq \left\langle \begin{cases} 0 & y_i<0 \\ y_i & y_i\in[0,1] \\ 1 & y_i>1\end{cases}\right\rangle_{i=1}^d\]
   )
   \STATE Take $n$ samples from each coin and let $\vec{a}$ be the empirical biases.
   \STATE {\bfseries return} $g(f(\vec{a}))$
\end{algorithmic}
\end{algorithm}
\begin{proof}
Note that when $\varepsilon \geq 1/2$, a trivial algorithm that outputs a vector with $1/2$ in each component works. Thus the most interesting case is when $\varepsilon < 1/2$.
Our list replicable algorithm is described in \autoref{dbep_algorithm}.

So we will prove its correctness by giving for each possible bias $\vec{b}\in[0,1]^d$, a set $L_{\vec{b}}$ 
with the three necessary properties: (1) $\abs{L_{\vec{b}}}\leq d+1$, (2) $L_{\vec{b}}\subseteq\ballcn{\epsilon}{\vec{b}}{\infty}$ 
(and also the problem specific restriction that $L_{\vec{b}}\subseteq[0,1]^d$), and (3) when given access to coins of biases $\vec{b}$, with probability at least $1-\delta$ the algorithm returns a value in $L_{\vec{b}}$.

Assume notation from \autoref{dbep_algorithm}. Let $L_{\vec{b}}=\set{g(f_\epsilon(\vec{x}))\colon \vec{x}\in\ballcinf{\epsilon_0}{\vec{b}}}$. By \autoref{universal_rounding_function}, $f_\epsilon$ takes on at most $d+1$ values on $\ballcinf{\epsilon_0}{\vec{b}}$ (which means $g\circ f_\epsilon$ also takes on at most $d+1$ values on this ball) which proves that $\abs{L_{\vec{b}}}\leq d+1$. This proves property (1).

NExt we state the following ibservation which says that the coordinate restriction function $g$ of \autoref{dbep_algorithm} does not reduce approximation quality. The proof is relatively straightforward, but tedious.

\begin{observation}\label{g_maintains_approximation}
Using the notation of \autoref{dbep_algorithm}, if $\vec{y}\in\ballcinf{\epsilon}{\vec{b}}$ then $g(\vec{y})\in\ballcinf{\epsilon}{\vec{b}}$.
\end{observation}

\begin{proof}
Let $\vec{z}=g(\vec{y})$. We must show for each $i\in[d]$ that $z_i\in[b_i-\epsilon,b_i+\epsilon]$. Note that
\[
z_i = 
\begin{cases}
    0 & y_i < 0 \\
    y_i & y_i\in[0,1] \\
    1 & y_i > 1
\end{cases}
\]
so we proceed with three cases.

\paragraph{Case 1: $y_i<0$.}
In this case, $z_i=0$ so because $b_i\in[0,1]$ we have $z_i\leq b_i+\epsilon$. 
Also, because $y_i\in[b_i-\epsilon,b_i+\epsilon]$, we have $b_i\leq y_i+\epsilon <0+\epsilon = z_i+\epsilon$, so subtracting $\epsilon$ from both sides gives $z_i > b_i-\epsilon$. 
Thus, we have $z_i\in[b_i-\epsilon,b_i+\epsilon]$ as desired.

\paragraph{Case 2: $y_i>1$.}
This case is symmetric to Case 1.

\paragraph{Case 3: $y_i\in[0,1]$.}
In this case $z_i=y_i\in[b_i-\epsilon,b_i+\epsilon]$ so we are done.
\end{proof}

We now establish Property (2). We know from \autoref{universal_rounding_function} that for each $\vec{x}\in\ballcinf{\epsilon_0}{\vec{b}}$ we have $f_\epsilon(\vec{x})\in\ballcinf{\epsilon}{\vec{b}}$, and by \autoref{g_maintains_approximation}, $g$ maintains this quality and we have $g(f_\epsilon(\vec{x}))\in\ballcinf{\epsilon}{\vec{b}}$. This shows that $L_{\vec{b}}\subseteq \ballcinf{\epsilon}{\vec{b}}$ proving property (2).

By Chernoff's bounds,  
for a single biased coin, $n=O\left(\frac{\ln(1/\delta_0)}{\epsilon_0^2}\right)$ independent samples of the coin is enough that with probability at least $1-\delta_0$, the empirical bias is within $\epsilon_0$ of the true bias. Thus, by a union bound, if we take $n$ samples of each of the $d$ coins, there is a probability of at most $d\cdot\delta_0=\delta$ that at least one of the empirical coin biases is not within $\epsilon_0$ of the true bias. Thus, by taking $n$ samples of each coin, we have with probability at least $1-\delta$ that the empirical biases $\vec{a}$ belong to $\ballcinf{\epsilon_0}{\vec{b}}$. In the case that this occurs, we have by definition of $L_{\vec{b}}$ that the value $g(f_\epsilon(\vec{a}))$ returned by the algorithm belongs to the set $L_{\vec{b}}$. This proves property (3).
\end{proof}

\begin{theorem}\label{thm:coin-certificate}
There is a $\lceil \log {d\over \delta}\rceil $-certificate replicable algorithm for \dbep\ with sample complexity of $n = O({d^2\over \varepsilon^2\delta^2})$ per coin.
\end{theorem}
\begin{proof}
Let $\varepsilon$ and $\delta$ be the input parameters to the algorithm and $\vec{b}$ the bias vector. Set $\varepsilon_0 = \frac{\varepsilon\delta}{2d}$. The algorithm $A$ first estimates the bias of each coin with up to $\varepsilon_0$ with a probability error parameter $\frac{\delta}{d}$ using a standard estimation algorithm.  Note that this can be done using $O({d^2\over \varepsilon^2\delta^2})$ tosses per coin. Let $\vec{v}$ be the output vector. It follows that $\vec{v} \in \ballcinf{\varepsilon_0}{\vec{b}}$ with probability at least $1-\delta$. Then it runs the deterministic function $f$ described in Lemma~\ref{lemma:certificate} with input $r\in \{0,1\}^{\ell}$ with $\ell = \lceil \log {d\over \delta} \rceil$ and $\vec{v}$ and outputs the value of the function. Lemma~\ref{lemma:certificate} guarantees that for $1-\delta$ fraction of the $r$s, all $\vec{v} \in\ballcinf{\varepsilon_0}{\vec{b}}$ gets rounded to the same value by $f$. Hence algorithm $A$ satisfies the requirements of the certificate-replicability. 
The certificate complexity is $\lceil \log {d\over \delta} \rceil$.
\end{proof}

Note that a $\ell$-certificate replicable leads to a $2^{\ell}$-list replicable algorithm. Thus Theorem~\ref{thm:coin-certificate} gives a $O(\frac{d}{\delta})$-list replicable algorithm for \dbep with sample complexity $O({d^2\over \varepsilon^2\delta^2})$. However, this is clearly sub-optimal and Theorem~\ref{yes_d_plus_1_pd_for_dbep} gives algorithms with a much smaller list and sample complexities.

\subsection{An Impossibility Result}

\begin{theorem}\label{no_d_pd_for_dbep}
For $k<d+1$, there does not exist a $k$-list replicable algorithm for the \dbep.

\end{theorem}

Before proving the theorem, we need a lemma whose proof appears in the appendix.

\begin{lemma}\label{distance_lemma}
For biases $\vec{a},\vec{b}\in[0,1]^d$ we have $\dtv\left(\distribution{A}{\vec{a}}{n},\distribution{A}{\vec{b}}{n}\right)\leq n\cdot d\cdot\norm{\vec{b}-\vec{a}}_\infty$.
\end{lemma}

\begin{proof}
We can view the model as algorithm $A$ having access to a single draw from a distribution. The distribution giving one sample flip of each coin in a collection with bias $\vec{b}$ is the $d$-fold product of Bernoulli distributions $\prod_{i=1}^d \bern(b_i)$ (which for notational brevity we denote as $\bern(\vec{b}$), so the distribution which gives $n$ independent flips of each coin is the $n$-fold product of this (using notation of \cite{a_survey_on_distribution_testing} written as $\bern(\vec{b})^{\otimes n}$).

Comparing the distributions of $n$ independent flips of the $d$ coins for bias $\vec{b}$ as compared to bias $\vec{a}$, we have for each $i\in[d]$ that
\[
\dtv\left(\bern(b_i), \bern(a_i)\right) = \abs{b_i-a_i}
\]
so by C.1.2 and C.1.3 of \cite{a_survey_on_distribution_testing} we have
\[
\dtv\left(\bern(\vec{b}), \bern(\vec{a})\right) \leq \sum_{i=1}^d \abs{b_i-a_i} \leq d\cdot \norm{\vec{b}-\vec{a}}_{\infty}
\]
and
\[
\dtv\left(\bern(\vec{b})^{\otimes n}, \bern(\vec{a})^{\otimes n}\right)  \leq n\cdot d\cdot \norm{\vec{b}-\vec{a}}_{\infty}.
\]

Because $A$ is a randomized function of one draw of this distribution, by D.1.2 of \cite{a_survey_on_distribution_testing} we have that $A$ cannot serve to increase the total variation distance, so
\[
\dtv\left(\distribution{A}{\vec{a}}{n},\distribution{A}{\vec{b}}{n}\right) \leq \dtv\left(\bern(\vec{b})^{\otimes n}, \bern(\vec{a})^{\otimes n}  \right)
\]
which completes the proof.
\end{proof}

\begin{proof}[Proof of \autoref{no_d_pd_for_dbep}]
Fix any $d\in\N$, and choose $\epsilon$ and $\delta$ as $\epsilon<\frac12$ and $\delta\leq\frac{1}{d+2}$.

Suppose for contradiction that such an algorithm $A$ does exists for some $k<d+1$. This means that for each possible threshold $\vec{t}\in[0,1]^d$, there exists some set $L_{\vec{t}}\subseteq\mathcal{H}$ of hypotheses with three properties: (1) each element of $L_{\vec{t}}$ is an $\epsilon$-approximation to $h_{\vec{t}}$, (2) $\abs{L_{\vec{t}}}\leq k$, and (3) with probability at least $1-\delta$, $A$ returns an element of $L_{\vec{t}}$.

Suppose for contradiction that such an algorithm does exist for some $k<d+1$. This means that for each possible bias $\vec{b}\in[0,1]^d$, there exists some set $L_{\vec{b}}\subseteq \ballcn{\epsilon}{\vec{b}}{\infty}$ (not necessarily unique, but consider some fixed one) with $\abs{L_{\vec{b}}}\leq k$ such that with probability at least least $\frac1k \cdot (1-\delta) \geq \frac1k \cdot (1-\frac1{d+2}) = \frac1k \cdot \frac{d+1}{d+2} \geq \frac1k \cdot \frac{k+1}{k+2}$, $A$ returns an element of $L_{\vec{b}}$. By the trivial averaging argument, this means that there exists at least one element in $L_{\vec{b}}$ which is returned by $A$ with probability at least $\frac1k \cdot \frac{k+1}{k+2}$. Let $f\colon[0,1]^d\to[0,1]^d$ be a function which maps each bias $\vec{b}$ to such an element of $L_{\vec{b}}$.

Since $\frac1k \cdot \frac{k+1}{k+2} > \frac{1}{k+1}$, let $\eta$ be such that $0 < \eta < \frac1k \cdot \frac{k+1}{k+2} - \frac{1}{k+1}$.

The function $f$ induces a partition $\P$ of $[0,1]^d$ where the members of $\P$ are the fibers of $f$ (i.e. $\P=\set{f^{-1}(\vec{y})\colon \vec{y}\in\range(f)}$). By definition, for any member $X\in\P$ there exists some $\vec{y}\in\range{f}$ such that for all $\vec{b}\in X$, $f(\vec{b})=\vec{y}$. By definition of $k$-pseudodeterministic $\epsilon$-approximation, we have $f(\vec{b})\in L_{\vec{b}}\subseteq\ballcn{\epsilon}{\vec{b}}{\infty}$ showing that $\vec{y}\in \ballcn{\epsilon}{\vec{b}}{\infty}$ and by symmetry $\vec{b}\in \ballcn{\epsilon}{\vec{y}}{\infty}$. This shows that $X\subseteq\ballcn{\epsilon}{\vec{y}}{\infty}$, so $\diam_\infty(X)\leq2\epsilon<1$.

Let $r=\frac{\eta\cdot d}{n}$. Since every member of $\P$ has $\ell_\infty$ diameter less than $1$, by \autoref{cubical_sperner_partition} there exists a point $\vec{p}\in[0,1]^d$ such that $\ballcn{r}{\vec{p}}{\infty}$ intersects at least $d+1>k$ members of $\P$. Let $\vec{b}^{(1)},\ldots,\vec{b}^{(d+1)}$ be points belonging to distinct members of $\P$ that all belong to $\ballcn{r}{\vec{p}}{\infty}$. By definition of $\P$, this means for distinct $j,j'\in[d+1]$ that $f(\vec{b}^{(j)})\not=f(\vec{b}^{(j')})$.

Now, for each $j\in[d+1]$, because $\norm{\vec{p}-\vec{b}^{(j)}}_\infty\leq r$, by \autoref{distance_lemma} we have $\dtv(\distribution{A}{\vec{p}}{n},\distribution{A}{\vec{b^{(j)}}}{n})\leq n\cdot d\cdot r=\eta$. However, this gives rise to a contradiction because the probability that $A$ with access to biased coins $\vec{b}^{(j)}$ returns $f(\vec{b}^{(j)})$ is at least $\frac1k \cdot \frac{k+1}{k+2}$, and by the total variation distance, it must be that $A$ with access to biased coins $\vec{p}$ returns $f(\vec{b}^{(j)})$ with probability at least $\frac1k \cdot \frac{k+1}{k+2} - \eta > \frac{1}{k+1}$; notationally, $\Pr_{\distribution{A}{\vec{b^{(j)}}}{n}}(\set{f(\vec{b}^{(j)})})\geq \frac1k \cdot \frac{k+1}{k+2}$ and $\dtv(\distribution{A}{\vec{b^{(j)}}}{n}, \distribution{A}{\vec{p}}{n})\leq \eta$, so $\Pr_{\distribution{A}{\vec{p}}{n}}(\set{f(\vec{b}^{(j)})})\geq \frac1k \cdot \frac{k+1}{k+2} - \eta > \frac{1}{k+1}$. This is a contradiction because a distribution cannot have $d+1\geq k+1$ disjoint events that each have probability greater than $\frac{1}{k+1}$.
\end{proof}

We conclude this section by noting that the above impossibility result implies a lower-bound on certificate complexity for \dbep. It follows that there is no $\lfloor \log (d)\rfloor$-certificate replicable algorithm for \dbep. In particular, any $l$-certificate replicable algorithm requires $l=\Omega(\log d/\\delta)$. Hence our algorithms for \dbep\ has optimal list and certificate complexity.

\section{Replicability in PAC Learning}

In this section, we establish replicability results for the PAC model. First, we define the PAC learning model.

Let $\HH$ be a (hypothesis) class of Boolean functions over $X$, and $\cal D$ be a distribution over $X$. For a function $f \in \cal H$, let ${\cal D}_f$ a distribution over $X \times \{0, 1\}$ that is obtained by sampling an element $x \in X$ according ${\cal D}$ and outputs $\langle x, f(x)\rangle$. For a hypothesis $h$, its error with respect to $\DD_f$ denoted by $e_{\DD_f}(h)$  is $\Pr_{\langle x,f(x)\rangle\in \DD_f}(f(x)\neq h(x))$.

\begin{definition}
A hypothesis class (or concept class) ${\cal H}$ is PAC learnable with sample complexity $m$, if there  is a learning algorithms $A$  with the following property: For every $f\in \mathcal{H}$, for every ${\cal D}$ a distribution over $X$, for all $0< \delta, \epsilon < 1$, $A$ on inputs $\delta,\epsilon$  and $S$ drawn i.i.d according to $\DD_f$ where $|S| \leq m$: outputs a hypothesis $h$ so that  $e_{\DD_f}(h) \leq \varepsilon$. with probability $\geq (1-\delta)$.
\end{definition}

We show that every hypothesis that can be learned via a statistical query learning algorithm has a reproducible PAC learning algorithm. We first define the notion of learning with statistical queries~\cite{kearns98}.

\begin{definition}
A statistical oracle $STAT({\cal D}_f, \nu)$ takes as an input a real-valued function $\phi: X \times \{0, 1\} \rightarrow (0, 1)$ and returns an estimate $v$ such that  
\[|v - E_{\langle x, y\rangle\in {\cal D}_f} [\phi(x, y)]| \leq \nu\]
\end{definition}

\begin{definition}
We say that an algorithm ${A}$ learns a concept class $\HH$ via statistical queries if for every  distribution ${\cal D}$ and every function $f \in \HH$,  for every $0 < \varepsilon < 1$,  there exists $\nu$ such that the algorithm $A$ on input $\varepsilon$, and $STAT(\DD_f, \nu)$ as an oracle, 
outputs a hypothesis $h$ such that $e_{D_f}(h) \leq \varepsilon$. The concept class is {\em non-adaptively} learnable if all the queries made by $A$ are non-adaptive.
\end{definition}

\subsection{Notions of Replicable PAC Learning}

We now define the notions of list and certificate replicable learning in the PAC model.

\begin{definition} Let $\HH$ be a hypothesis class. We say that $\HH$ is  $k$-{\em list replicably learnable} if there is an algorithm $A$ with the following properties. 
For every $f\in \mathcal{H}$ and for every distribution $\DD$  over $X$, for every  $0< \epsilon < 1$ and $0< \delta \leq 1$, there is a list $L$ of size at most $k$ consisting of $\varepsilon$-approximate hypotheses such that  
 $A$ on inputs $\delta,\epsilon$  and samples $S\in \DD_f^m$,  $\Pr_{S \sim \DD_f^m} [A(S) \in L] \geq 1-\delta$. We call $m$ the {\em sample complexity} of the learning algorithm and $k$ to be its {\em list complexity}.
\end{definition}

Next we define {\em certificate replicability}. This is very close to the notion of replicability given in ~\cite{impagliazzo_replicability_2022}. However, our main concern is the amount of randomness needed by the algorithm to make it perfectly reproducible. 

\begin{definition}

 Let $\HH$ be a concept class. We say that $\HH$ is $\ell$-certificate replicably learnable if the following holds. There is a learning algorithm $A$ such that for every $f \in \HH$, for every distribution $\DD$ over $X$, $A$ gets the following inputs: $\epsilon$, $\delta$, $r \in \{0, 1\}^{\ell}$, and $S \sim \DD_f^m$. 
 \[\Pr_{r \in \{0, 1\}^\ell}\left[ \exists h_r : \Pr_{s\sim\mathcal{D}^m}[A(s;r)=h_r] \geq 1-\delta\right] \geq 1-\delta\]
We call $m$ the {\em sample complexity} and the number of random  bits $\ell$ the {\em certificate complexity} of $A$. 
\end{definition}

The definition can be further explained as follows. For the algorithm $A$, we say $r$ is a ``certificate of replicability" if $A$ with $m$ independent samples from $\DD_f$, and  $r$ as input, outputs a canonical $\varepsilon$-approximate hypothesis $h_r$ with probability $\geq 1-\delta$. Then the above definition demands that $1-\delta$ fraction of $r\in \{0,1\}^l$ are certificates of replicability.

\subsection{Replicable Algorithms}

\begin{theorem}\label{thm:list-stat}
Let $\HH$ be a concept class that is learnable with $d$ non-adaptive statistical queries, then $\HH$ is $(d+1)$-list reproducibly learnable. Furthermore, the sample complexity $n=n(\nu,\delta)$ of the $(d+1)$-list replicable algorithm equals  $O(d^2\log {d\over \delta} \cdot {1\over \nu^2})$, where $\nu$ is the  approximation error parameter  of each statistical query oracle.
\end{theorem}

\begin{proof}
The proof is very similar to the proof of Theorem~\ref{yes_d_plus_1_pd_for_dbep}. 
Our replicable algorithm $B$ works as follows. Let $\varepsilon$ and $\delta$ be input parameters and $\DD$ be a distribution and $f\in \HH$. 
Let $A$ be the statistical query learning algorithm for $\HH$ that outputs a hypothesis $h$  with approximation error $e_{\DD_f}(h) = \varepsilon$. Let $STAT(D_f,\nu)$ be the statistical query oracle for this algorithm. Let $\phi_1,\ldots,\phi_d$ be the statistical queries made by $A$.

 Let $\vec{b} = \langle E_{\langle x, y\rangle \in \DD_f}[\phi_1(\langle x, y\rangle), \ldots, E_{\langle x, y\rangle \in \DD_f}[\phi_d(\langle x, y\rangle)\rangle$. Set $\varepsilon_0 = \frac{\nu}{2d}$.  The algorithm $B$ first estimates the values  $b[i]=E_{\langle x, y\rangle \in \DD_f}[\phi_i(\langle x, y\rangle)]$, $1 \leq i \leq d$ upto an approximation error of $\epsilon_0$ with success probably $1-\delta/d$ for each query.  
  Note that this can be done by a simple empirical estimation algorithm, that uses a total of $n = O(\frac{d^2}{\nu^2}\cdot \log d/\delta)$ samples.  
 Let $\vec{v}$ be the estimated vector.
It follows that $\vec{v} \in \ballcinf{\varepsilon_0}{\vec{b}}$ with probability at least $1-\delta$.

Now the algorithm $B$ evaluates  the deterministic function $f_\epsilon$ from Lemma~\ref{universal_rounding_function} on input $\vec{v}$.  Let $\vec{u}$ be the output vector.  Now the algorithm $B$ simulates the statistical query algorithm $A$ with $\vec{u}[i]$ as the answer to the query $\phi_i$. By Lemma~\ref{universal_rounding_function}, $\vec{u} \in \ballcinf{\nu}{\vec{b}}$. Thus the error of the hypothesis output by the algorithm is at most $\epsilon$. Since $A$ is a deterministic algorithm the number of possible outputs only depends on the number of outputs of the function $f_{\varepsilon}$, more precisely the number of possible outputs is the size of the set  $\{f_{\varepsilon}(\vec{v}) : v \in \ballcinf{\varepsilon_0}{\vec{b}}\}$ which is almost $d+1$, by Lemma~\ref{universal_rounding_function}. 
Thus the total number of possible outputs of the algorithm $B$ is at most $d+1$ with probability at least $1-\delta$.
\end{proof}

We note that we can simulate a statistical query algorithm that makes $d$ {\em adaptive} queries to get a $2^d$-list replicable learning algorithm. This can be done by rounding each query  to two possible values (the approximation factor increases by 2). The sample complexity of this algorithm will be $O({d\over \nu^2}\cdot\log {1\over \delta})$. The sample complexity can be improved to  $\tilde{O}({\sqrt{d}\over \nu^2})$ by using techniques from adaptive data analysis~\cite{BNSSSU21}. 

Next, we design a certificate replicable algorithm for hypothesis classes that admit statistical query learning algorithms.  The proof the following theorem is in the appendix.

\begin{theorem}\label{thm:cert-stat}
Let $\HH$ be a concept class that is learnable with $d$ non-adaptive statistical queries, then $\HH$ is $\lceil \log {d\over \delta}\rceil$-certificate reproducibly learnable. Furthermore, the sample complexity $n=n(\nu,\delta)$ of this algorithm equals  $O(\frac{d^2}{\nu^2\delta^2}\cdot \log {d\over \delta})$, where $\nu$ is the  approximation error parameter  of each statistical query oracle.
\end{theorem}

\begin{proof}
The proof is very similar to the proof of Theorem~\ref{thm:coin-certificate}. 
Our replicable algorithm $B$ works as follows, let $\varepsilon$ and $\delta$ be input parameters and $\DD$ be a distribution and $f\in \HH$. 
Let $A$ be the statistical query learning algorithm for $\HH$ that outputs a hypothesis $h$  with approximation error $e_{\DD_f}(h) = \varepsilon$. Let $STAT(D_f,\nu)$ be the statistical query oracle for this algorithm. Let $\phi_1,\ldots,\phi_d$ be the statistical queries made by $A$.

 Let $\vec{b} = \langle E_{\langle x, y\rangle \in \DD_f}[\phi_1(\langle x, y\rangle), \ldots, E_{\langle x, y\rangle \in \DD_f}[\phi_d(\langle x, y\rangle)\rangle$. Set $\varepsilon_0 = \frac{\nu\delta}{2d}$. The algorithm $B$ first estimates the values  $b[i]=E_{\langle x, y\rangle \in \DD_f}[\phi_i(\langle x, y\rangle)]$, $1 \leq i \leq d$ upto an approximation error of $\epsilon_0$ with success probably $1-\delta/d$ for each query.  Note that this can be done by a simple empirical estimation algorithm, that uses a total of $n = O(\frac{d^2}{\nu^2\delta^2}\cdot \log d/\delta)$ samples.  
 Let $\vec{v}$ be the estimated the vector.
It follows that $\vec{v} \in \ballcinf{\varepsilon_0}{\vec{b}}$ with probability at least $1-\delta$. 
Now the algorithm $B$ evaluates  the deterministic function $f$ described in Lemma~\ref{lemma:certificate} with inputs $r\in \{0,1\}^{\ell}$ where $\ell = \lceil \log {d\over \delta} \rceil$ and $\vec{v}$. By Lemma~\ref{lemma:certificate} for at least $1-\delta$ fraction of the $r$'s , the function $f$ outputs a canonical $\vec{v^*} \in \ballcinf{\nu}{\vec{b}}$. 
Now the algorithm $B$ simulates the statistical query algorithm $A$ with $\vec{v*}[i]$ as the answer to the query $\phi_i$. Since $A$ is a deterministic algorithm it follows that  our algorithm $B$ is certificate replicable. Finally, note that 
the certificate complexity is $\lceil \log {d\over \delta} \rceil$.
\end{proof}

As before we consider the case when the statistical query algorithm makes $d$ adaptive queries. The proof of the following theorem appears in the appendix.

\begin{theorem}
Let $\HH$ be a concept class that is learnable with $d$ adaptive statistical queries, then $\HH$ is $\lceil d \log {d\over \delta}\rceil$-certificate reproducibly learnable. Furthermore, the sample complexity of this algorithm equals  $O(\frac{d^3}{\nu^2\delta^2}\cdot \log {d\over \delta})$, where $\nu$ is the  approximation error parameter  of each statistical query oracle.
\end{theorem}

\begin{proof}[Proof Sketch]
The proof uses similar arguments as before. The main difference we will evaluate each query with an approximation error of $\frac{\nu\delta}{d}$ with a probability error of $d/\delta$. This requires $O(\frac{d^2}{\nu^2\delta^2}\cdot \log {d\over \delta})$ per query. We use a fresh set of certificate randomness for each such evaluation. Note that the length of the certificate for each query is $\lceil \log d/\delta\rceil$. Thus the total certificate complexity is $\lceil d \log {d\over \delta}\rceil$.
\end{proof}

\subsection{Impossibility Results in the PAC Model}

In this section, we establish matching upper and lower bounds for  the \dtep\ 
in the PAC model with respect to the uniform distribution. We establish that this problem admits a $(d+1)$-list replicable algorithm and does not admit a $d$-list replicable algorithm.

\begin{problem}[\dtep]\label{dtep}
Fix some $d\in\N$. Let $X=[0,1]^d$. For each value $\vec{t}\in [0,1]^d$ (which happens to be the same as $X$), let $h_{\vec{t}}:X\to\set{0,1}$ be the function defined by 
\[
h_{\vec{t}}(\vec{x})=
\begin{cases} 
    1 & \text{for each $i\in[d]$, it holds that $x_i\leq t_i$} \\ 
    0 & \text{otherwise}
\end{cases}.
\]
This is the function that determines if each coordinate is less than or equal to the thresholds specified by $\vec{t}$. Let $\mathcal{H}$ be the hypothesis class consisting of all such threshold functions: $\mathcal{H}=\set{h_{\vec{t}}\;|\;\vec{t}\in [0,1]^d}$.
\end{problem}

We first observe the impossibility of list-replicable algorithms in the general PAC model. This follows from known results.

\begin{observation}\label{thm:littlestone-nolist}
Let $k$ be any constant. There is no $k$-list replicable algorithm for the \dtep\ in the PAC model even when $d=1$.
\end{observation}

\begin{proof}
From the works of~\cite{BLM20} and~\cite{ALMM19}, it follows that any class has finite Littlestone dimension if and only if there exists a constant $k$ such that the concept class has a $k$-list replicable algorithm in the PAC model. Since the concept class \dtep has infinite Littlestone dimension even when $d=1$, the theorem follows.    
\end{proof}

The above result rules out list-replicable algorithms in the general PAC model.
In the rest of this section, we study the replicable learnability of \dtep\ in the PAC model under the uniform distribution. We establish matching upper and  lower bounds on the list complexity.

\begin{theorem}
    In the PAC model under the uniform distribution, there is a $d+1$-list replicable algorithm for \dtep
\end{theorem}

\begin{proof}
    It is known and easy to see that \dtep\ is learnable under the uniform distribution by making $d$ nonadaptive statistical queries. Thus by Theorem~\ref{thm:list-stat}, \dtep\ admits a $(d+1)$-list replicable algorithm.
\end{proof}

We next establish that the above result is tight by proving that  there is no $d$-list replicable algorithm in the PAC model under the uniform distribution. 

\begin{theorem}\label{no_d_pd_for_dtep}
For $k<d+1$, there does not exist a $k$-list replicable algorithm for the \dtep\ in the PAC model.
\end{theorem}

 The proof that for $k<d+1$, there is no algorithm which $k$-list reproducibly learns \dtep\ in the PAC model is similar to the proof of \autoref{no_d_pd_for_dbep}. The reason is that sampling $d$-many biased coins with biases $\vec{b}$ is similar to obtaining a point $\vec{x}$ uniformly at random from $[0,1]^d$ and evaluating the threshold function $h_{\vec{b}}$ on it---this corresponds to asking whether all of the coins were heads/$1$'s. The two models differ though, because in the sample model for the \dbep, the algorithm sees for each coin whether it is heads or tails, but this information is not available in the PAC model for the \dtep. Conversely, in the PAC model for the \dtep, a random draw from $[0,1]^d$ is available to the algorithm, but in the sample model for the \dbep\ the algorithm does not get this information.

Furthermore, there is the following additional complexity in the impossibility result for the \dtep. In the \dbep, 
we said by definition that a collection of $d$ coins parameterized by bias vector $\vec{a}$ was an $\epsilon$-approximation to a collection of $d$ coins parameterized by bias vector $\vec{b}$ if and only if $\norm{\vec{b}-\vec{a}}_{\infty}\leq\epsilon$, and we used this norm in applying the results of \cite{geometry_of_rounding_arxiv}. However, the notion of $\epsilon$-approximation in the PAC model is quite different than this. It is possible to have a hypotheses $h_{\vec{a}}$ and $h_{\vec{b}}$ in the \dtep such that $\norm{\vec{b}-\vec{a}}_{\infty}>\epsilon$ but with respect to some distribution $\mathcal{D}_X$ on the domain $X$ we have $\err_{\mathcal{D}_X}(h_{\vec{a}},h_{\vec{b}})\leq\epsilon$. For example, if $\mathcal{D}_X$ is the uniform distribution on $X=[0,1]^d$ and $\vec{a}=\vec{0}$ and $\vec{b}$ is the first standard basis vector $\vec{b}=\langle 1,0,\ldots,0\rangle$, and $\epsilon=\frac12$, then $\norm{\vec{b}-\vec{a}}_{\infty}=1>\epsilon$, but $\err_{\mathcal{D}_X}(h_{\vec{a}},h_{\vec{b}})=0\leq\epsilon$ because $h_{\vec{a}}(\vec{x})\not=h_{\vec{b}}(\vec{x})$ if and only if all of the last $d-1$ coordinates of $\vec{x}$ are $0$ and the first coordinate is $>0$, but there is probability $0$ of sampling such $\vec{x}$ from the uniform distribution on $X=[0,1]^d$.

For this reason, we can't just partition $[0,1]^d$ as we did with the proof of \autoref{no_d_pd_for_dbep} and must do something more clever. It turns out that it is possible to find a subset $[\alpha,1]^d$ on which hypotheses parameterized by vectors on opposite faces of this cube $[\alpha,1]^d$ have high PAC error between them. A consequence by the triangle inequality of $\err_{\mathcal{D}_X}$ is that two such hypotheses cannot both be approximated by a common third hypothesis. That is what the following lemma states.

\begin{restatable}{lemma}{RestatableThresholdApproxLemma}\label{threshold-approx-lemma}
Let $d\in\N$ and $\alpha=\frac{d-1}{d}$. Let $\vec{s},\vec{t}\in[\alpha,1]^d$ such that there exists a coordinate $i_0\in[d]$ where $s_{i_0}=\alpha$ and $t_{i_0}=1$ (i.e. $\vec{s}$ and $\vec{t}$ are on opposite faces of this cube). Let $\epsilon\leq\frac{1}{8d}$. Then there is no point $\vec{r}\in X$ such that both $\errunif(h_{\vec{s}},h_{\vec{r}})\leq\epsilon$ and $\errunif(h_{\vec{t}},h_{\vec{r}})\leq\epsilon$ (i.e. there is no hypothesis which is an $\epsilon$-approximation to both $h_{\vec{s}}$ and $h_{\vec{t}}$).
\end{restatable}

\begin{proof}
Let $\vec{q}=\left\langle \begin{cases} s_i & i=i_0 \\ t_i & i\not=i_0\end{cases}\right\rangle_{i=1}^d$ which will serve as a proxy to $\vec{s}$. 

We need the following claim.

\begin{claim}
For each $\vec{x}\in X$, the following are equivalent:
\begin{enumerate}
    \item\label{item:claim-neq} $h_{\vec{q}}(\vec{x})\not=h_{\vec{t}}(\vec{x})$
    \item\label{item:claim-01} $h_{\vec{q}}(\vec{x})=0$ and $h_{\vec{t}}(\vec{x})=1$
    \item\label{item:claim-xi} $x_{i_0}\in(q_{i_0},t_{i_0}]=(\alpha,1]$ and for all $i\in[d]\setminus\set{i_0}$, $x_i\in[0,t_i]$.
\end{enumerate}
Furthermore, the above equivalent conditions imply the following:
\begin{enumerate}[resume] 
    \item\label{item:claim-s} $h_{\vec{s}}(\vec{x})\not=h_{\vec{t}}(\vec{x})$.
\end{enumerate}
\end{claim}
\begin{proof}[Proof of Claim]~

(\ref{item:claim-01}) $\Longrightarrow$ (\ref{item:claim-neq}):
This is trivial.

(\ref{item:claim-neq}) $\Longrightarrow$ (\ref{item:claim-01}):

Note that because $q_{i_0}=s_{i_0}=\alpha<1=t_{i_0}$, we have for all $i\in[d]$ that $q_i\leq t_i$. If $h_{\vec{t}}(\vec{x})=0$ then for some $i_1\in[d]$ it must be that $x_{i_1}>t_{i_1}$, but since $t_{i_1}\geq q_{i_1}$ it would also be the case that $x_{i_1}>q_{i_1}$, so $h_{\vec{q}}(\vec{x})=0$ which gives the contradiction that $h_{\vec{q}}(\vec{x})=h_{\vec{t}}(\vec{x})$. Thus $h_{\vec{t}}(\vec{x})=1$, and since $h_{\vec{q}}(\vec{x})\not=h_{\vec{t}}(\vec{x})$ we have $h_{\vec{q}}(\vec{x})=0$.

(\ref{item:claim-neq}) $\iff$ (\ref{item:claim-xi}):
We partition $[0,1]^d$ into three sets and examine these three cases.

Case 1: $x_{i_0}\in(q_{i_0},t_{i_0}]=(\alpha,1]$ and for all $i\in[d]\setminus\set{i_0}$, $x_i\in[0,t_i]$. 
In this case, $q_{i_0}<x_{i_0}$ so $h_{\vec{q}}(\vec{x})=0$ and for all $i\in[d]$ $x_i\leq t_i$, so $h_{\vec{t}}(\vec{x})=1$, so $h_{\vec{q}}(\vec{x})\not=h_{\vec{t}}(\vec{x})$.

Case 2: $x_{i_0}\not\in(q_{i_0},t_{i_0}]=(\alpha,1]$ and for all $i\in[d]\setminus\set{i_0}$, $x_i\in[0,t_i]$. 
In this case, because $x_{i_0}\in[0,1]$ and $x_{i_0}\not\in(\alpha,1]$ we have $x_{i_0}\leq\alpha=q_{i_0}\leq t_{i_0}$ and also for all other $i\in[d]\setminus\set{i_0}$, $x_i\leq t_i=q_i$ (by definition of $\vec{q}$). Thus $h_{\vec{q}}(\vec{x})=1=h_{\vec{t}}(\vec{x})$.

Case 3: For some $i_1\in[d]\setminus\set{i_0}$, $x_{i_1}\not\in[0,t_{i_1}]$. 
In this case, because $x_{i_1}\in[0,1]$, we have $x_{i_1}>t_{i_1}=q_{i_1}$. Thus $h_{\vec{q}}(\vec{x})=0=h_{\vec{t}}(\vec{x})$.

Thus, it is the case that $h_{\vec{q}}(\vec{x})\not=h_{\vec{t}}(\vec{x})$ if and only if $x_{i_0}\in(q_{i_0},t_{i_0}]=(\alpha,1]$ and for all $i\in[d]\setminus\set{i_0}$, $x_i\in[0,t_i]$.

(\ref{item:claim-neq}, \ref{item:claim-01}, \ref{item:claim-xi}) $\Longrightarrow$ (\ref{item:claim-s}):
By (\ref{item:claim-01}), we have $x_{i_0}>q_{i_0}$, and since $q_{i_0}=s_{i_0}$ by definition of $\vec{q}$, it follows that $x_{i_0}>s_{i_0}$ which means $h_{\vec{s}}(\vec{x})=0$. By (\ref{item:claim-xi}), $h_{\vec{t}}(\vec{x})=1$ which gives $h_{\vec{s}}(\vec{x})\not=h_{\vec{t}}(\vec{x})$.
\end{proof}

With this claim in hand, our next step will be two prove the following two inequalities: \[
2\epsilon < \errunif(h_{\vec{q}},h_{\vec{t}}) \leq \errunif(h_{\vec{s}},h_{\vec{t}}).
\]

For the second of these inequalities, note that by the (\ref{item:claim-neq}) $\Longrightarrow$ (\ref{item:claim-s}) part of claim above, since $h_{\vec{q}}(\vec{x})\not=h_{\vec{t}}(\vec{x})$ implies $h_{\vec{s}}(\vec{x})\not=h_{\vec{t}}(\vec{x})$ we have
\begin{align*}
    \errunif(h_{\vec{q}},h_{\vec{t}})
    &= \Pr_{\vec{x}\dist\mathrm{unif}(X)}[h_{\vec{q}}(\vec{x})\not=h_{\vec{t}}(\vec{x})] \\
    &\leq \Pr_{\vec{x}\dist\mathrm{unif}(X)}[h_{\vec{s}}(\vec{x})\not=h_{\vec{t}}(\vec{x})] \\
    &= \errunif(h_{\vec{s}},h_{\vec{t}}).
\end{align*}

Now, for the first of the inequalities above, we will use the (\ref{item:claim-neq}) $\iff$ (\ref{item:claim-xi}) portion of the claim, we will use our hypothesis that $\vec{t}\in[\alpha,1]^d$ (which implies for each $i\in[d]$ that $[0,t_i]\subseteq[0,\alpha]$), we will use the hypothesis that $\epsilon\leq\frac{1}{8d}$, and we will use \autoref{alpha-lemma}. Utilizing these, we get the following:
\begin{align*}
    &\errunif(h_{\vec{q}},h_{\vec{t}}) \\
    &= \Pr_{\vec{x}\dist\mathrm{unif}(X)}[h_{\vec{q}}(\vec{x})\not=h_{\vec{t}}(\vec{x})] \\
    &= \Pr_{\vec{x}\dist\mathrm{unif}(X)}[x_{i_0}\in(\alpha,1]\;\wedge\;\forall i\in[d]\setminus\set{i_0},\,x_i\in[0,t_i]] \\
    &= \Pr_{x_{i_0}\dist\mathrm{unif}([0,1])}[x_{i_0}\in(\alpha,1]]
       \cdot
       \prod_{\substack{i=1\\i\not=i_0}}^d\Pr_{x\dist\mathrm{unif}([0,1])}[x\in[0,t_i]] \\
    &\geq \Pr_{x_{i_0}\dist\mathrm{unif}([0,1])}[x_{i_0}\in(\alpha,1]]
       \cdot
      \prod_{\substack{i=1\\i\not=i_0}}^d\Pr_{x\dist\mathrm{unif}([0,1])}[x\in[0,\alpha]] \\
    &= (1-\alpha)\cdot\alpha^{d-1} \\
    &> \frac1{4d} \\
    &\geq 2\epsilon.
\end{align*}

Thus, we get the desired two inequalities:
\[
2\epsilon < \errunif(h_{\vec{q}},h_{\vec{t}}) \leq \errunif(h_{\vec{s}},h_{\vec{t}}).
\]
This nearly completes the proof. If there existed some point $\vec{r}\in X$ such that both $\errunif(h_{\vec{s}},h_{\vec{r}})\leq\epsilon$ and $\errunif(h_{\vec{t}},h_{\vec{r}})\leq\epsilon$, then it would follow from the triangle inequality of $\errunif$ that
\[
\errunif(h_{\vec{s}},h_{\vec{t}}) \leq \errunif(h_{\vec{s}},h_{\vec{r}}) + \errunif(h_{\vec{t}},h_{\vec{r}}) \leq 2\epsilon
\]
but this would contradict the above inequalities, so no such $\vec{r}$ exists.
\end{proof}

Equipped with the above lemma, we are now ready to prove Theorem~\ref{no_d_pd_for_dtep}.

\begin{proof}{of Theorem~\ref{no_d_pd_for_dtep}}
Fix any $d\in\N$, and choose $\epsilon$ and $\delta$ as $\epsilon\leq\frac1{4d}$ and $\delta\leq\frac{1}{d+2}$.
We will use the constant $\alpha=\frac{d-1}{d}$ and consider the cube $[\alpha,1]^d$. We will also consider only the uniform distribution over $X$.

Suppose for contradiction that such an algorithm $A$ does exists for some $k<d+1$. This means that for each possible threshold $\vec{t}\in[0,1]^d$, there exists some set $L_{\vec{t}}\subseteq\mathcal{H}$ of hypotheses with three properties: (1) each element of $L_{\vec{t}}$ is an $\epsilon$-approximation to $h_{\vec{t}}$, (2) $\abs{L_{\vec{t}}}\leq k$, and (3) with probability at least $1-\delta$, $A$ returns an element of $L_{\vec{t}}$.

By the trivial averaging argument, this means that there exists at least one element in $L_{\vec{t}}$ which is returned by $A$ with probability at least $\frac1k \cdot (1-\delta) \geq \frac1k \cdot (1-\frac1{d+2}) = \frac1k \cdot \frac{d+1}{d+2} \geq \frac1k \cdot \frac{k+1}{k+2}$. Let $f\colon[\alpha,1]^d\to[0,1]^d$ be a function which maps each threshold $\vec{t}\in[\alpha,1]^d$ to such an element of $L_{\vec{t}}$. This is slightly different from the proof of \autoref{no_d_pd_for_dbep} because we are defining the function $f$ on only a very specific subset of the possible thresholds. The reason for this was alluded to in the discussion following the statement of \autoref{no_d_pd_for_dtep}.

Since $\frac1k \cdot \frac{k+1}{k+2} > \frac{1}{k+1}$, let $\eta$ be such that $0 < \eta < \frac1k \cdot \frac{k+1}{k+2} - \frac{1}{k+1}$.

The function $f$ induces a partition $\P$ of $[\alpha,1]^d$ where the members of $\P$ are the fibers of $f$ (i.e. $\P=\set{f^{-1}(\vec{y})\colon \vec{y}\in\range(f)}$). For any member $W\in\P$ and any coordinate $i\in[d]$, it cannot be that the set ${w_i\colon \vec{w}\in W}$ contains both values $\alpha$ and $1$---if it did, then there would be two points $\vec{s},\vec{t}\in W$ such that $s_i=\alpha$ and $t_i=1$, but because they both belong to $W$, there is some $\vec{y}\in[0,1]^d$ such that $f(\vec{s})=\vec{y}=f(\vec{t})$, but by definition of the partition, $h_{\vec{y}}$ would have to be an $\epsilon$-approximation (in the PAC model) of both $h_{\vec{s}}$ and $h_{\vec{t}}$, but by \autoref{no_d_pd_for_dtep}, this is not possible.

Thus, the partition $\P$ is a ``non-spanning'' partition of $[\alpha,1]^d$ as in the proof of \autoref{cubical_sperner_partition}, so there is some point $\vec{p}\in[\alpha,1]^d$ such that for every radius $r>0$, it holds that $\ballcinf{r}{\vec{p}}$ intersects at least $d+1$ members of $\P$.

Similar to \autoref{distance_lemma} and how it is used in the proof of \autoref{no_d_pd_for_dbep}, we can use the following two facts. First, the
function $\gamma_1$ defined by $\gamma_1(\vec{s},\vec{t})=\errunif(h_{\vec{s}},h_{\vec{t}})$ is continuous (with respect to the $\ell_\infty$ norm on the domain). Second, the function $\gamma_2(h_{\vec{s}},h_{\vec{t}})=\dtv(\distribution{A}{\vec{s}}{n},\distribution{A}{\vec{t}}{n})$ is continuous (with respect to the $\errunif$ notion of distance on the domain). A consequence is that the composition $\gamma_{12}(\vec{s},\vec{t})=\dtv(\distribution{A}{\vec{s}}{n},\distribution{A}{\vec{t}}{n})$ is continuous. Thus, we can find some radius $r>0$ such that if $\norm{\vec{t}-\vec{s}}_\infty\leq r$, then $\dtv(\distribution{A}{\vec{s}}{n},\distribution{A}{\vec{t}}{n})\leq\eta$.

Now we get the same type of contradiction as in the proof of \autoref{no_d_pd_for_dbep}: for the special point $\vec{p}$ we have that $\distribution{A}{\vec{p}}{n}$ is a distribution that has $d+1\geq k+1$ disjoint events that each have probability greater than $\frac{1}{k+1}$. Thus, no $k$-list replicable algorithm exists.

\end{proof}

\section{Conclusions}
In this work, we investigated the pressing issue of replicability in machine learning from an algorithmic point of view. We observed that replicability in the absolute sense is difficult to achieve. Hence we considered  two natural extensions that capture the degree of (non) replicability: list and certificate replicability. We designed replicable algorithms with a small list, certificate, and sample complexities for the \dbep\ and the class of problems that can be learned via statistical query algorithms that make non-adaptive statistical queries. We also established certain impossibility results in the PAC model of learning and for  \dbep. There are several interesting research directions that emerge from our work. There is a gap in the sample complexities of the list and certificate reproducibilities with comparable parameters. Is this gap inevitable? Currently, there is an exponential gap in the replicability parameters between hypothesis classes that can be learned via non-adaptive  and adaptive statistical queries. Is this gap necessary? A generic question is to explore the trade-offs  between the sample complexities, list complexity, certificate complexities, adaptivity, and nonadaptivity.

\section{Acknowledgements}
We thank an anonymous reviewer for suggesting Observation~\ref{thm:littlestone-nolist}.
Pavan's work is partly supported by NSF award 2130536. Part of the work was done when Pavan was visiting Simons Institute for the Theory of Computing. Vander Woude and Vinodchandran's work is partly supported by NSF award 2130608.








\bibliography{references}
\bibliographystyle{alphaurl}

\appendix

\end{document}